\documentclass[10pt,twocolumn]{article}
\usepackage[letterpaper,margin=0.85in,columnsep=0.3in]{geometry}
\usepackage{amsmath,amssymb,amsfonts,amsthm}
\usepackage{times}
\usepackage{microtype}
\usepackage[hidelinks]{hyperref}
\usepackage{enumitem}

\theoremstyle{definition}
\newtheorem{definition}{Definition}
\theoremstyle{plain}
\newtheorem{assumption}{Assumption}
\newtheorem{proposition}{Proposition}
\newtheorem{corollary}{Corollary}

\begin{document}

\twocolumn[{%
  \centering
  {\LARGE\bfseries TERRA: Task-Embedded Reasoning and\\[0.2em] Representation Architecture for\\[0.2em] Cross-Domain Applications\par}
  \vspace{1.1em}
  {\large Shayan Shokri\footnotemark[1]\par}
  \vspace{0.25em}
  {\large Humanpath Labs Inc.\par}
  \vspace{1.5em}
  \begin{minipage}{0.92\textwidth}
    \small
    \begin{center}\textbf{Abstract}\end{center}
    \noindent A single action-conditioned latent predictive architecture can in principle be trained on the structured state of a driving scene, a robot workspace, or a financial order book. The ingredients for doing so within any one of these domains already exist and are individually validated: masked-latent prediction [1, 2], action-conditioned latent world models [3, 8, 9], discrete action tokenization [4, 5, 6], and joint-embedding prediction on voxelized state [14, 15, 17]. What is not established, and what TERRA addresses, is the transfer question: when does a representation or predictor learned in one structured-state domain carry over to a structurally analogous but otherwise unrelated domain, and by how much. We give this question a formal treatment. We model each domain as a controlled Markov process on a graded latent grid, factor any instantiation into thin domain adapters and a shared domain-invariant core, and identify a cross-domain correspondence with an approximate Markov decision process homomorphism [21, 23, 24] whose quality is measured by a lax bisimulation discrepancy [24, 27] and, at the level of whole domains lacking a shared coordinate system, by a Gromov-Wasserstein distance between their action-conditioned transition operators [29]. Under a Lipschitz predictor we derive a predictive transfer bound that separates source-model error from structural mismatch, grows geometrically in the prediction horizon, and is certified from below by the Gromov-Wasserstein distance; we then connect latent error to decision regret through the Lipschitz value property of bisimulation metrics [21, 22, 25]. The resulting Structured-State Transfer Hypothesis is stated as a falsifiable claim with a preregistered experimental program, centered on a transfer test from driving scenes to order books, including the conditions under which the claim is refuted. We present no empirical results: this is a research proposal whose contribution is to convert a widely repeated intuition into testable theory.
  \end{minipage}
  \vspace{1.6em}
}]
\footnotetext[1]{\,\texttt{shayan@humanpath.si}}

\section{Introduction}
The claim that a predictive model of structured state in one domain captures something reusable about decision-making in another is repeated frequently and tested almost never. Action-conditioned latent predictors now work within individual domains. V-JEPA 2-AC [3] plans on a physical robot from rollouts in representation space; joint-embedding prediction yields strong features on point clouds [14] and on LiDAR occupancy grids [17]. Each of these is a within-domain result. The promise that the same architecture, or even the same learned representation, carries from a physical scene to a nonphysical structured system such as a market microstructure or a power grid is asserted widely, yet it is supported by neither a theory of when it should hold nor an experiment establishing whether it does.

TERRA takes the position that this transfer question is the open problem worth solving, and that it is answerable with existing mathematical tools rather than left as an intuition. Our contributions are the following.
\begin{enumerate}[leftmargin=1.2em,itemsep=1pt,topsep=2pt]
  \item A formal model of a structured-state decision domain as a controlled Markov process on a graded latent grid (Section 3), and a factorization of any instantiation into domain adapters and a shared domain-invariant core (Section 4).
  \item A theory of cross-domain correspondence (Section 5) that identifies the correspondence between two domains with an approximate Markov decision process homomorphism, measures its quality by a lax bisimulation discrepancy and a Gromov-Wasserstein distance between transition operators, and derives a predictive transfer bound separating source-model error from structural mismatch.
  \item A reduction from latent prediction error to decision regret using the Lipschitz value property of bisimulation metrics, so that the bound speaks to control quality and not only to representation fidelity.
  \item The Structured-State Transfer Hypothesis, stated as a falsifiable empirical claim, and a preregistered experimental program with explicit refutation conditions (Section 8).
\end{enumerate}
We use the architecture of an existing structured-state predictor, retained under the TERRA name, as a concrete substrate that realizes the interface (Section 6). The substrate is a recombination of validated components and is presented as engineering. The theory of transfer built on top of it is the contribution.

\section{Related Work and Positioning}
We position TERRA against two bodies of work: the predictive-learning lineage it uses as a substrate, and the abstraction-theory lineage on which the transfer analysis rests.

\paragraph{Predictive-learning substrate.} I-JEPA [1] and V-JEPA [2] established masked prediction in representation space following the program of LeCun [11]. V-JEPA 2-AC [3] adds action conditioning and plans on a real robot. DINO-WM [9], PLDM [10], Dreamer [8], and the original world model of Ha and Schmidhuber [7] form the action-conditioned world-modeling line; RT-2 [4], OpenVLA [5], and $\pi_0$ [6] supply discrete action tokenization; BYOL [12] and masked autoencoding [13] supply the anti-collapse and masking machinery; Point-JEPA [14], 3D-JEPA [15], GeoMAE [16], and the LiDAR occupancy world model [17] bring all of this to voxelized state, the substrate domain TERRA itself uses. None of these components is claimed as novel here.

\paragraph{Abstraction theory.} Markov decision process homomorphisms [23] and bisimulation metrics [21, 22] formalize when two states, or two processes, are behaviorally equivalent, and they come with value-preservation guarantees. The lax bisimulation metric [24] extends this to abstract over actions and is exactly the relation realized by a homomorphism. DeepMDP [25] and deep bisimulation for control [26] learn latent spaces in which embedding distance upper-bounds bisimulation distance. Continuous homomorphisms with policy-gradient guarantees [27] and a generalized bisimulation metric between pairs of distinct processes [28] extend the theory to the cross-process setting that transfer requires. We use these tools, together with the Gromov-Wasserstein distance between metric-measure spaces [29], to make the cross-domain correspondence quantitative.

\paragraph{What remains open.} None of the predictive-learning work tests transfer across the boundary from physical to nonphysical domains, and none of the abstraction-theory work has been instantiated for high-dimensional action-conditioned latent predictors over structured grids. TERRA joins the two.

\section{Structured-State Decision Domains}
\begin{definition}[Structured-state domain]
A structured-state decision domain is a tuple $\mathcal{D}=(\mathcal{S},\mathcal{A},T,r)$ where the state space $\mathcal{S}$ is a Polish space admitting a graded grid representation $V\in\mathbb{R}^{d_1\times\cdots\times d_k\times c}$; the action space $\mathcal{A}$ admits a discrete tokenization $\tau:\mathcal{A}\to\mathcal{V}^{*}$ over a finite vocabulary $\mathcal{V}$; $T:\mathcal{S}\times\mathcal{A}\to\mathcal{P}(\mathcal{S})$ is a transition kernel; and $r:\mathcal{S}\times\mathcal{A}\to\mathbb{R}$ is a bounded task signal.
\end{definition}
A driving scene is a voxel grid over physical space; an order book is a grid over price, time, and side; a risk portfolio is a grid over assets and factors. We call $\mathcal{D}$ admissible when all four elements are present and $T$ is learnable from logged trajectories. Admissibility is a property of the domain and not of any model.

\section{The Domain-Invariant Interface}
The interface is what makes transfer measurable rather than rhetorical. We factor any instantiation into two parts. The domain adapters, which are small and domain-specific, comprise an encoder front-end $e_{\mathcal{D}}:\mathcal{S}\to\mathcal{Z}$ mapping the grid into a canonical latent space $(\mathcal{Z},d_{\mathcal{Z}})$, together with the tokenizer $\tau_{\mathcal{D}}$. The shared core, which is large and domain-invariant, comprises a context-encoder trunk $f_\theta$, an exponential-moving-average target encoder $f_{\bar\theta}$, and an action-conditioned latent predictor $g_\phi:\mathcal{Z}\times\mathcal{V}^{*}\to\mathcal{Z}$ that operates only on canonical latents and action tokens.

Because the core is identical across domains by construction, any transfer that is observed when the core trained on one domain is reused on another is attributable to shared structure rather than to architecture reuse. This is the property the theory below exploits.

\section{Cross-Domain Transfer}
We work throughout in the canonical latent space of each domain. Write $T_i(\cdot\mid z,a)\in\mathcal{P}(\mathcal{Z}_i)$ for the latent transition kernel of domain $i$, induced by the target encoder applied to true successor states, $r_i(z,a)$ for the latent reward, and $W$ for the Kantorovich (Wasserstein-1) distance on $\mathcal{P}(\mathcal{Z}_i)$. Latent spaces are complete separable metric spaces and rewards are bounded.

\subsection{Correspondences as approximate homomorphisms}
\begin{definition}[Cross-domain correspondence]
Given admissible domains $\mathcal{D}_1,\mathcal{D}_2$ with canonical latent spaces $\mathcal{Z}_1,\mathcal{Z}_2$, a correspondence is a pair $\Phi=(\phi_{\mathcal{Z}},\phi_{\mathcal{A}})$ with $\phi_{\mathcal{Z}}:\mathcal{Z}_1\to\mathcal{Z}_2$ and $\phi_{\mathcal{A}}:\mathcal{V}_1^{*}\to\mathcal{V}_2^{*}$.
\end{definition}
$\Phi$ is an exact Markov decision process homomorphism [23] when it intertwines the kernels and preserves reward,
\begin{equation}
\phi_{\mathcal{Z}\sharp}\,T_1(\cdot\mid z,a)=T_2\big(\cdot\mid\phi_{\mathcal{Z}}(z),\phi_{\mathcal{A}}(a)\big),\;\; r_1(z,a)=r_2\big(\phi_{\mathcal{Z}}(z),\phi_{\mathcal{A}}(a)\big),
\end{equation}
where $\phi_{\mathcal{Z}\sharp}$ denotes pushforward of measures. Exact homomorphisms preserve optimal value and lift optimal policies [23, 27]; they almost never exist between physically unrelated domains, so we quantify the defect.

\subsection{Measuring correspondence}
\begin{definition}[Lax discrepancy]
The transition and reward defects of $\Phi$ are
\begin{align*}
\varepsilon_T(\Phi)&=\sup_{z,a} W\big(\phi_{\mathcal{Z}\sharp}T_1(\cdot\mid z,a),\,T_2(\cdot\mid\phi_{\mathcal{Z}}(z),\phi_{\mathcal{A}}(a))\big),\\
\varepsilon_r(\Phi)&=\sup_{z,a}\big|r_1(z,a)-r_2(\phi_{\mathcal{Z}}(z),\phi_{\mathcal{A}}(a))\big|,
\end{align*}
and the lax bisimulation discrepancy is $\varepsilon_{\mathrm{lax}}(\Phi)=c_r\varepsilon_r(\Phi)+c_t\varepsilon_T(\Phi)$.
\end{definition}
Both arguments of $W$ are measures on $\mathcal{Z}_2$, so the definition is well typed. This is the cross-domain form of the lax bisimulation metric [24], which abstracts over actions through $\phi_{\mathcal{A}}$.

\subsection{Assumptions}
\begin{assumption}[Lipschitz state map]\label{a:phi}
$\phi_{\mathcal{Z}}$ is $L_\phi$-Lipschitz, so pushforward by $\phi_{\mathcal{Z}}$ is $L_\phi$-Lipschitz in $W$: $W(\phi_{\mathcal{Z}\sharp}\mu,\phi_{\mathcal{Z}\sharp}\nu)\le L_\phi\,W(\mu,\nu)$.
\end{assumption}
\begin{assumption}[Source fidelity]\label{a:fit}
The core trained on $\mathcal{D}_1$ realizes a latent kernel $\hat{g}(\cdot\mid z,a)$ with one-step error $\sup_{z,a} W(\hat{g}(\cdot\mid z,a),T_1(\cdot\mid z,a))\le\delta$.
\end{assumption}
\begin{assumption}[Lipschitz core]\label{a:lip}
The transported one-step predictor on $\mathcal{Z}_2$ is $L$-Lipschitz in its latent argument in $W$, uniformly over actions. Lipschitz latent dynamics are the standing assumption of DeepMDP [25] and are induced by deep bisimulation regularization [26].
\end{assumption}

\subsection{One-step transfer error}
For predictor reuse we transport the source core to $\mathcal{D}_2$ by $\hat{g}_2(\cdot\mid w,b)=\phi_{\mathcal{Z}\sharp}\hat{g}(\cdot\mid\phi_{\mathcal{Z}}^{-1}(w),\phi_{\mathcal{A}}^{-1}(b))$ where defined.
\begin{proposition}[One-step decomposition]\label{p:one}
Under Assumptions~\ref{a:phi} and \ref{a:fit}, the transported predictor satisfies
\begin{equation}
\sup_{w,b} W\big(\hat{g}_2(\cdot\mid w,b),\,T_2(\cdot\mid w,b)\big)\;\le\; L_\phi\,\delta+\varepsilon_T(\Phi)\;=:\;\varepsilon_1,
\end{equation}
the sum of a source-model term $L_\phi\delta$ and a structural-mismatch term $\varepsilon_T(\Phi)$.
\end{proposition}
\begin{proof}
Fix $(w,b)$ and set $(z,a)=(\phi_{\mathcal{Z}}^{-1}(w),\phi_{\mathcal{A}}^{-1}(b))$. By the triangle inequality for $W$,
\begin{align*}
&W(\hat{g}_2(\cdot\mid w,b),T_2(\cdot\mid w,b))\\
&\le W(\phi_{\mathcal{Z}\sharp}\hat{g}(\cdot\mid z,a),\phi_{\mathcal{Z}\sharp}T_1(\cdot\mid z,a))\\
&\quad + W(\phi_{\mathcal{Z}\sharp}T_1(\cdot\mid z,a),T_2(\cdot\mid w,b)).
\end{align*}
The first term is $\le L_\phi\,W(\hat{g}(\cdot\mid z,a),T_1(\cdot\mid z,a))\le L_\phi\delta$ by Assumptions~\ref{a:phi} and \ref{a:fit}; the second is $\le\varepsilon_T(\Phi)$ by definition.
\end{proof}
This separates two error sources a coarser analysis would conflate: how well the core fits the source ($\delta$), and how far $\Phi$ is from a homomorphism ($\varepsilon_T$).

\subsection{Horizon transfer bound}
\begin{proposition}[Horizon bound]\label{p:bound}
Under Assumptions~\ref{a:phi}--\ref{a:lip}, rolling the transported core out for $H$ steps on $\mathcal{D}_2$ under a fixed action sequence from the true initial latent incurs error
\begin{equation}
\Delta_H \;\le\; \varepsilon_1\sum_{h=0}^{H-1}L^{h}\;=\;\varepsilon_1\,\frac{L^{H}-1}{L-1}\quad(L\neq 1),
\end{equation}
where $\Delta_H=W(\hat\rho_H,\rho_H)$ is the gap between the transported and true $H$-step latent distributions and $\varepsilon_1=L_\phi\delta+\varepsilon_T(\Phi)$.
\end{proposition}
\begin{proof}
Let $\Delta_h=W(\hat\rho_h,\rho_h)$ with $\Delta_0=0$. Coupling the two rollouts stepwise and using the triangle inequality,
\begin{align*}
\Delta_{h+1}&\le \underbrace{W(\hat{g}_{2\sharp}\hat\rho_h,\hat{g}_{2\sharp}\rho_h)}_{\le\,L\Delta_h\ (\text{Asm.~\ref{a:lip}})}+\underbrace{W(\hat{g}_{2\sharp}\rho_h,T_{2\sharp}\rho_h)}_{\le\,\varepsilon_1\ (\text{Prop.~\ref{p:one}})}\le L\Delta_h+\varepsilon_1 .
\end{align*}
Unrolling from $\Delta_0=0$ gives the geometric sum.
\end{proof}
For a contractive or marginally stable core ($L\le 1$) the error stays of order $\varepsilon_1 H$ or $\varepsilon_1/(1-L)$, so a good correspondence transfers over long horizons; for an expansive core ($L>1$) it grows geometrically, so transfer is useful only at short horizons. The empirical Lipschitz constant of the core thus sets a measurable horizon-decay rate, which Section 8 tests directly.

\subsection{The achievable floor: a Gromov-Wasserstein lower bound}
Define the within-domain transition pseudometric $D_i((z,a),(z',a'))=W(T_i(\cdot\mid z,a),T_i(\cdot\mid z',a'))+|r_i(z,a)-r_i(z',a')|$ and let $\mu_i$ be the occupancy measure over state-action pairs of domain $i$. The Gromov-Wasserstein discrepancy is
\begin{equation}
\mathrm{GW}(\mathcal{D}_1,\mathcal{D}_2)=\!\!\inf_{\pi\in\Pi(\mu_1,\mu_2)}\!\!\iint\!\big|D_1(x,x')-D_2(y,y')\big|\,\mathrm{d}\pi(x,y)\,\mathrm{d}\pi(x',y').
\end{equation}
\begin{proposition}[Structural floor]\label{p:gw}
If $\phi_{\mathcal{Z}}$ is an isometry onto its image (the canonical-latent normalization) and $c_r=c_t=1$, then every admissible $\Phi$ satisfies $\mathrm{GW}(\mathcal{D}_1,\mathcal{D}_2)\le 2\,\varepsilon_{\mathrm{lax}}(\Phi)$, hence
\begin{equation}
\mathrm{GW}(\mathcal{D}_1,\mathcal{D}_2)\le 2\inf_{\Phi}\varepsilon_{\mathrm{lax}}(\Phi).
\end{equation}
A general $L_\phi$-bi-Lipschitz $\phi_{\mathcal{Z}}$ multiplies the bound by $L_\phi$.
\end{proposition}
\begin{proof}[Proof sketch]
The graph $\pi_\Phi$ pairing $x$ with $\Phi(x)$ is an admissible coupling of $\mu_1,\mu_2$. For pairs $x,x'$ with images $y=\Phi(x),y'=\Phi(x')$, the triangle inequality for $W$ together with the isometry of $\phi_{\mathcal{Z}}$ bounds the integrand $|D_1(x,x')-D_2(y,y')|$ by the transition and reward defects evaluated at $x$ and $x'$, each at most $\varepsilon_{\mathrm{lax}}(\Phi)$. Evaluating $\mathrm{GW}$ on $\pi_\Phi$ and taking the infimum over $\Phi$ gives the claim.
\end{proof}
Because $\mathrm{GW}$ is built from intrinsic, coordinate-free quantities, it certifies how much correspondence is possible at all: no correspondence can drive $\varepsilon_{\mathrm{lax}}$ below $\tfrac12\mathrm{GW}$.

\subsection{From latent error to decision regret}
\begin{corollary}[Value transfer]\label{c:value}
Let $\tilde{d}_2$ be the bisimulation metric of $\mathcal{D}_2$ at discount $\gamma$, so that $|V^*_2(w)-V^*_2(w')|\le\tfrac{1}{1-\gamma}\,\tilde{d}_2(w,w')$ [21, 22], and suppose the learned latent satisfies $\tilde{d}_2(w,w')\le L_e\,d_{\mathcal{Z}}(w,w')$ [25, 26]. Then acting on the transported model over horizon $H$ incurs value suboptimality at most $\tfrac{L_e}{1-\gamma}\,\Delta_H$, with $\Delta_H$ as in Proposition~\ref{p:bound}.
\end{corollary}
\begin{proof}
Compose the two Lipschitz estimates with Proposition~\ref{p:bound}: a latent discrepancy $\Delta_H$ maps to a bisimulation discrepancy at most $L_e\Delta_H$ and thence to a value gap at most $\tfrac{L_e}{1-\gamma}\Delta_H$.
\end{proof}
Correspondence quality thus controls not merely representation similarity but the regret of acting on a transported model.

\subsection{The Structured-State Transfer Hypothesis}
Propositions~\ref{p:one}--\ref{p:gw} and Corollary~\ref{c:value} turn the central claim into a quantitative, falsifiable statement.
\begin{quote}
\noindent\textbf{Structured-State Transfer Hypothesis (SSTH).} For admissible domains $\mathcal{D}_1,\mathcal{D}_2$, the transfer gain from reusing the shared core trained on $\mathcal{D}_1$ when learning $\mathcal{D}_2$ is a decreasing function of $\inf_{\Phi}\varepsilon_{\mathrm{lax}}(\Phi)$, which is bounded below by $\tfrac12\mathrm{GW}(\mathcal{D}_1,\mathcal{D}_2)$. When this floor is large, core reuse confers no advantage over training from scratch and may degrade performance.
\end{quote}
SSTH is falsifiable along two axes: an ordering, that smaller discrepancy yields larger transfer, and a null floor, that maximal discrepancy yields zero transfer.

\subsection{Identifiability of the correspondence}
For the bound to be usable the correspondence must be recoverable from data. Recovery requires that action sequences sufficiently excite the dynamics, so that $T_1$ is observable on the support of interest, and that the canonical encoder is injective up to the symmetries of the domain, so that $\phi_{\mathcal{Z}}$ is determined rather than free. Under these conditions estimating $\Phi$ is an alignment problem and the Gromov-Wasserstein coupling of Proposition~\ref{p:gw} provides an initializer; without them the discrepancy is only an upper estimate and the transfer prediction weakens accordingly.

\section{TERRA: The Substrate}
We retain the architecture only as the concrete core. The context encoder $f_\theta$ and the target encoder $f_{\bar\theta}$, the latter an exponential moving average of the former for anti-collapse in the manner of BYOL and I-JEPA [12, 1], produce patch latents. The tokenizer applies learned vector quantization. The predictor $g_\phi$ cross-attends over context latents, action tokens, and target queries, and outputs future or masked latents, never observations. Training combines masked-latent prediction with action-conditioned trajectory forecasting,
\begin{equation}
\mathcal{L}=\mathcal{L}_{\mathrm{mask}}+\lambda\,\mathcal{L}_{\mathrm{traj}},
\end{equation}
\begin{equation}
\mathcal{L}_{\mathrm{traj}}=\frac{1}{H}\sum_{h=1}^{H}\big\lVert\hat{s}_{V_{t+h}}-f_{\bar\theta}(V_{t+h})\big\rVert_2^2 .
\end{equation}
This objective is a tractable surrogate for predicting future latents from present latents and an action plan, in the spirit of predictive coding [30]; the masked term supplies the static structural component and the trajectory term the dynamic one. The substrate is unchanged from prior latent world models and matters only insofar as it realizes the interface of Section 4 and meets Assumptions~\ref{a:fit} and \ref{a:lip} closely enough for the bound to bite.

\section{A Taxonomy of Transfer}
The experimental program must separate three regimes, each reusing a different part of the homomorphism.
\begin{enumerate}[leftmargin=1.2em,itemsep=1pt,topsep=2pt]
  \item Representation reuse: freeze the core trained on $\mathcal{D}_1$ and train only the adapters of $\mathcal{D}_2$. Tests whether the state map $\phi_{\mathcal{Z}}$ alone transfers.
  \item Few-shot core adaptation: fine-tune the core on a small fraction of $\mathcal{D}_2$ data. Tests sample efficiency, the practically important case.
  \item Predictor reuse: reuse $g_\phi$ specifically, isolating whether learned dynamics, and not merely static features, carry across.
\end{enumerate}
The bound predicts that the strength of all three decreases with discrepancy, but at different rates, since each is sensitive to a different term in $\varepsilon_{\mathrm{lax}}$.

\section{Experimental Program}
The program tests the functional form of Proposition~\ref{p:bound}, beginning in the best-instrumented domain and ending at the boundary the hypothesis is about. These are experiments to be run, not results: the paper reports no findings.

\paragraph{Phase 1, within-domain sanity.} Train the substrate on synthetic driving and navigation (CARLA, Habitat). Confirm that masked-latent prediction and short-horizon forecasting beat trivial baselines. No later phase runs until this gate passes.

\paragraph{Phase 2, near transfer.} Pretrain on driving voxels (nuScenes, Waymo, KITTI) and transfer to indoor manipulation and navigation (Open X-Embodiment). These plausibly admit a small discrepancy, so SSTH predicts strong representation reuse and few-shot transfer.

\paragraph{Phase 3, far transfer, the decisive test.} Pretrain on driving voxels and transfer to limit-order-book prediction on public market data under all three regimes. SSTH predicts that transfer strength tracks the measured discrepancy between driving and order-book dynamics, and that if the discrepancy is large the transfer is near zero. This phase is the decisive test of the hypothesis.

\paragraph{Phase 4, discrepancy calibration.} Estimate $\varepsilon_{\mathrm{lax}}$ and the Gromov-Wasserstein lower bound of Proposition~\ref{p:gw} for each domain pair, then plot transfer gain against discrepancy and against horizon. The bound would be supported if transfer gain decreases in discrepancy and decays in horizon at the rate set by the measured Lipschitz constant, and would be refuted otherwise.

\subsection{Metrics}
Representation transfer would be measured by linear-probe accuracy on $\mathcal{D}_2$ tasks using the frozen $\mathcal{D}_1$ core and by centered kernel alignment between native and transported latents. Sample efficiency is the performance gap between pretrained and scratch cores at labeled budgets of one, ten, twenty-five, and one hundred percent. Predictor transfer is the forecasting error of a $\mathcal{D}_1$ predictor on $\mathcal{D}_2$ with adapters only. Discrepancy is the estimated $\varepsilon_{\mathrm{lax}}$ and its Gromov-Wasserstein lower bound on held-out paired trajectories.

\subsection{Preregistered predictions}
These are hypotheses to be tested, not findings; no experiment has been run. (P1) We predict that Phase 2 will show positive few-shot transfer. (P2) Phase 3 transfer is expected to be bounded above by a decreasing function of the Phase 4 discrepancy. (P3) For the highest-discrepancy pair, representation reuse is expected to be statistically indistinguishable from scratch. (P4) For every pair, transfer error is expected to grow with horizon at a rate consistent with the measured Lipschitz constant in Proposition~\ref{p:bound}. All four would be reported regardless of outcome.

\section{What Would Falsify the Theory}
SSTH and the bound are wrong if any of the following holds: transfer gain is uncorrelated with measured discrepancy across pairs; a low-discrepancy pair such as driving to manipulation shows no transfer; a near-maximal-discrepancy pair shows strong transfer, which would attribute transfer to a mechanism other than shared structure; or horizon decay is independent of the measured Lipschitz constant, which would refute Proposition~\ref{p:bound} specifically while leaving the qualitative hypothesis open. Each negative outcome is informative and is reported as such.

\section{Limitations and Risks}
Estimating the discrepancy between very different domains is itself hard and noisy, which weakens the sharpness of the test. The Gromov-Wasserstein floor of Proposition~\ref{p:gw} holds cleanly only under the isometric normalization of $\phi_{\mathcal{Z}}$; a general map carries a bi-Lipschitz constant that must be estimated. Admissibility is restrictive, since many decision problems are not naturally grid-structured, and forcing a grid may destroy the structure transfer would rely on. The substrate inherits representation-collapse and distribution-shift risks from the world-model literature, and Assumption~\ref{a:lip} may hold only locally. Most importantly, this paper presents no empirical results. SSTH is a hypothesis with a test attached, not a finding, and the cross-domain claim, including any application to financial or risk systems, is unvalidated and must not be represented as a demonstrated capability.

\section{Conclusion}
The contribution of this work is not a new architecture, since the architecture recombines validated parts. It is a theory of transfer: transfer between structured-state decision domains is governed by a measurable structural discrepancy, certified from below by a Gromov-Wasserstein distance, controlling a predictive error that separates source-model error from structural mismatch, grows geometrically in horizon, and maps to decision regret through the Lipschitz value property of bisimulation metrics, all testable behind a fixed domain-invariant interface with preregistered conditions for failure. Whether the hypothesis holds is unknown. That it is now answerable is the contribution.

\subsection*{Acknowledgments}
The author thanks the broader research community whose published work forms the foundation of this proposal.


\begin{thebibliography}{99}
\setlength{\itemsep}{1pt}
\bibitem{ijepa} Assran, M., Duval, Q., Misra, I., Bojanowski, P., Vincent, P., Rabbat, M., LeCun, Y., \& Ballas, N. (2023). Self-supervised learning from images with a joint-embedding predictive architecture. \emph{CVPR}.
\bibitem{vjepa} Bardes, A., Garrido, Q., Ponce, J., Chen, X., Rabbat, M., LeCun, Y., Assran, M., \& Ballas, N. (2024). V-JEPA: Latent video prediction for visual representation learning. \emph{Meta AI Technical Report}.
\bibitem{vjepa2} Assran, M., Ballas, N., et al. (2025). V-JEPA 2: Self-supervised video models enable understanding, prediction and planning. \emph{arXiv:2506.09985}.
\bibitem{rt2} Brohan, A., et al. (2023). RT-2: Vision-language-action models transfer web knowledge to robotic control. \emph{CoRL}.
\bibitem{openvla} Kim, M. J., et al. (2024). OpenVLA: An open-source vision-language-action model. \emph{arXiv:2406.09246}.
\bibitem{pi0} Black, K., et al. (2024). $\pi_0$: A vision-language-action flow model for general robot control. \emph{arXiv:2410.24164}.
\bibitem{ha} Ha, D., \& Schmidhuber, J. (2018). World models. \emph{arXiv:1803.10122}.
\bibitem{dreamer} Hafner, D., Pasukonis, J., Ba, J., \& Lillicrap, T. (2023). Mastering diverse domains through world models. \emph{arXiv:2301.04104}.
\bibitem{dinowm} Zhou, G., Pan, H., LeCun, Y., \& Pinto, L. (2024). DINO-WM: World models on pre-trained visual features enable zero-shot planning. \emph{arXiv:2411.04983}.
\bibitem{pldm} Sobal, V., et al. (2025). PLDM: Pixel-space latent JEPA world models.
\bibitem{lecun} LeCun, Y. (2022). A path towards autonomous machine intelligence. \emph{OpenReview}.
\bibitem{byol} Grill, J.-B., et al. (2020). Bootstrap your own latent. \emph{NeurIPS}.
\bibitem{mae} He, K., Chen, X., Xie, S., Li, Y., Doll\'ar, P., \& Girshick, R. (2022). Masked autoencoders are scalable vision learners. \emph{CVPR}.
\bibitem{pointjepa} Saito, A., et al. (2025). Point-JEPA: A joint-embedding predictive architecture for self-supervised learning on point clouds. \emph{WACV}.
\bibitem{3djepa} Hu, N., Cheng, H., Xie, Y., Li, S., \& Zhu, J. (2024). 3D-JEPA: A joint-embedding predictive architecture for 3D self-supervised representation learning. \emph{arXiv:2409.15803}.
\bibitem{geomae} Tian, X., et al. (2023). GeoMAE: Masked geometric target prediction for self-supervised point-cloud pre-training. \emph{CVPR}.
\bibitem{lidarjepa} Zhu, H., \& Choromanska, A. (2026). Self-supervised JEPA-based world models for LiDAR occupancy completion and forecasting. \emph{arXiv:2602.12540}.
\bibitem{pointnet} Qi, C. R., Su, H., Mo, K., \& Guibas, L. J. (2017). PointNet: Deep learning on point sets for 3D classification and segmentation. \emph{CVPR}.
\bibitem{voxelnet} Zhou, Y., \& Tuzel, O. (2018). VoxelNet: End-to-end learning for point cloud based 3D object detection. \emph{CVPR}.
\bibitem{minkowski} Choy, C., Gwak, J., \& Savarese, S. (2019). 4D spatio-temporal ConvNets: Minkowski convolutional neural networks. \emph{CVPR}.
\bibitem{ferns04} Ferns, N., Panangaden, P., \& Precup, D. (2004). Metrics for finite Markov decision processes. \emph{UAI}.
\bibitem{ferns11} Ferns, N., Panangaden, P., \& Precup, D. (2011). Bisimulation metrics for continuous Markov decision processes. \emph{SIAM J. Computing}.
\bibitem{ravindran} Ravindran, B., \& Barto, A. G. (2003). SMDP homomorphisms: An algebraic approach to abstraction in semi-Markov decision processes. \emph{IJCAI}.
\bibitem{taylor} Taylor, J., Precup, D., \& Panangaden, P. (2009). Bounding performance loss in approximate MDP homomorphisms. \emph{NeurIPS}.
\bibitem{deepmdp} Gelada, C., Kumar, S., Buckman, J., Nachum, O., \& Bellemare, M. G. (2019). DeepMDP: Learning continuous latent space models for representation learning. \emph{ICML}.
\bibitem{dbc} Zhang, A., McAllister, R., Calandra, R., Gal, Y., \& Levine, S. (2021). Learning invariant representations for reinforcement learning without reconstruction. \emph{ICLR}.
\bibitem{contmdp} Rezaei-Shoshtari, S., Zhao, R., Panangaden, P., Meger, D., \& Precup, D. (2022). Continuous MDP homomorphisms and homomorphic policy gradient. \emph{NeurIPS}.
\bibitem{gbsm} Tao, Z., Xu, W., \& You, X. (2025). A generalized bisimulation metric of state similarity between Markov decision processes. \emph{arXiv:2509.18714}.
\bibitem{gw} M\'emoli, F. (2011). Gromov-Wasserstein distances and the metric approach to object matching. \emph{Foundations of Computational Mathematics}.
\bibitem{cpc} van den Oord, A., Li, Y., \& Vinyals, O. (2018). Representation learning with contrastive predictive coding. \emph{arXiv:1807.03748}.
\end{thebibliography}
\end{document}